\documentclass[11pt]{article}

\usepackage{authblk}
\usepackage{setspace}
\usepackage{cprotect}
\usepackage{fullpage}
\usepackage[utf8]{inputenc} 
\usepackage[T1]{fontenc}    
\usepackage{hyperref}       
\usepackage{url}            
\usepackage{booktabs}       
\usepackage{amsfonts}       
\usepackage{nicefrac}       
\usepackage{microtype}      
\usepackage{graphicx}
\usepackage{subfigure}
\usepackage{booktabs} 
\usepackage{caption}
\usepackage{url}            
\usepackage{amsfonts}       
\usepackage{amsmath}
\usepackage{amssymb}
\usepackage{siunitx}
\usepackage{bbm}
\usepackage{multirow}
\usepackage{pifont}
\usepackage{hhline}
\usepackage{bm}
\usepackage{xcolor,colortbl}
\usepackage{mathtools}
\definecolor{Gray}{gray}{0.85}
\usepackage{makecell}
\usepackage{relsize}
\usepackage[colorinlistoftodos]{todonotes}

\newif\ifpaper
\papertrue

\def\b0{{0}}

\def\RR{\mathbb{R}}

\def\>{\rangle}

\def\diag{\operatorname{\mathop{diag}}}

\def\Set#1{\left\{ #1 \right\}}

\newcommand{\E}{\mathbb{E}}
\newcommand{\Econd}[2]{\mathbb{E}\left[#1\;\middle\vert\;#2\right]}
\newcommand{\bigE}[1]{\mathbb{E}\left[#1\right]}
\newcommand{\distas}[1]{\mathbin{\overset{#1}{\sim}}}

\newtheorem{theorem}{Theorem}[section]
\newtheorem{lemma}[theorem]{Lemma}

\newenvironment{proof}{\par\noindent{\bf Proof:\ }}{\hfill$\Box$\\[2mm]}

\newcommand{\bigOmg}[1]{\Omega\!\left(#1\right)}

\newcommand{\inner}[1]{\left\langle#1\right\rangle}

\newcommand{\norm}[1]{\left\|#1\right\|}

\newcommand{\GF}[1]{\Gamma\left(#1\right)}

\def\min{\mathop{\rm min}\nolimits}
\def\max{\mathop{\rm max}\nolimits}

\def\var{\mathrm{var}}

\title{A Fully Rigorous Proof of the Derivation of Xavier and He's Initialization for Deep ReLU Networks}
\date{} 

\author{Quynh Nguyen\thanks{Email: quynhnguyenngoc89@gmail.com}
}

\begin{document}

\maketitle

\begin{abstract}
    A fully rigorous proof of the derivation of Xavier/He's initialization for ReLU nets is given.
\end{abstract}

\section{Introduction}
Consider an $L$-layer ReLU network with layer widths $(n_l)_{l=0}^{L}.$
Here, $n_0$ and $n_L$ denote the input and output dimension respectively, and the others are hidden layer widths.
For simplicity, we assume that the network has a single output, i.e. $n_L=1.$
Let the feature map $f_l:\RR^{n_0}\to\RR^{n_l}$ be defined as
\begin{align}\label{eq:def_feature_map}
    f_l(x)=\begin{cases}
	    x & l=0,\\
	    \sigma(W_l^T f_{l-1}) & l\in[L-1],\\
	    W_L^T f_{L-1} & l=L, 
        \end{cases}
\end{align}
where $W_l\in\RR^{n_{l-1}\times n_l}$, and $\sigma(x)=\max(0,x).$ 
Let $g_l:\RR^{n_0}\to\RR^{n_l}$ be the pre-activation feature map so that $f_l(x)=\sigma(g_l(x)).$
Let us denote the backward derivative 
\begin{align}
    &\delta_{k,p}(x)=\frac{\partial f_L(x)}{\partial g_{k,p}(x)},\quad \forall\,k\in[L-1], p\in[n_k].
\end{align}
Two quantities that are relevant for deriving Xavier and He's initialization \cite{XavierBengio2010,he2015delving} 
are the variance of the output of every neuron (forward pass), and the variance of the above derivative (backward pass).
It is easy to see that for i.i.d\ weights at the initialization, we have
$\var(\delta_{k,p}(x))=\var(\delta_{k,q}(x))$ and $\var(f_{k,p}(x))=\var(f_{k,q}(x))$, for every $p,q\in[n_k], x\in\RR^{n_0}.$
For this reason, in the following we will write $\delta_{k,p}(x)$ without mentioning the value of $p$.
For every $k\in[L]$, let us also define
\begin{align}\label{eq:Sk}
	&S_k = \frac{\norm{x}_2}{\sqrt{2\pi}} \left(\prod_{l=1}^{k-1}\sqrt{\frac{n_l}{2}}\right) \left(\prod_{l=1}^{k}\beta_{l}\right).
\end{align}


In \cite{he2015delving,XavierBengio2010}, the authors propose to initialize neural network weights in such a way that 
the following properties hold at the initialization:
\begin{enumerate}
    \item All the neurons have the same variance:
    \begin{align}
	\var(f_{k,p}(x))=\var(f_{k-1,p}(x)), \quad\forall\,k\in[2,L]
    \end{align}
    \item All the backward derivatives have the same variance for all the neurons:
    \begin{align}
	\var\left(\delta_{k,p}(x)\right)=\var\left(\delta_{k-1,p}(x)\right), \quad\forall\,k\in[2,L]
    \end{align}
\end{enumerate}
By using a heuristic derivation, they end up with the following initialization for ReLU networks: 
$\beta_k=\frac{2}{n_{k-1}}$, or $\beta_k=\frac{2}{n_{k}}.$
Here, the choice of $\beta_k$ depends on which of the above two criteria is used.

\section{Main Results}
The following theorem provides a rigorous proof of the derivation of Xavier/He's initialization schemes \cite{XavierBengio2010,he2015delving}.
\begin{theorem}\label{thm:main}
    Let $S_k$ be defined as in \eqref{eq:Sk}.
    Then, we have:
    \begin{enumerate}
	\item 
	Fix any $\epsilon\in(0,1).$ 
	Suppose that $\min_{l\in[k-1]}n_l\geq \bigOmg{\frac{k}{\log(1+\epsilon)}}.$
	Then, it holds:
	\begin{align}
	    \left(\pi-(1+\epsilon)^2\right) S_k^2 \leq \var(f_{k,p}(x))\leq \left(\pi-(1-\epsilon)^2\right) S_k^2.
	\end{align}
	Moreover, $S_k=S_{k-1}$ if and only if $\beta_k^2=\frac{2}{n_{k-1}}.$
	\item $\var\left(\delta_{k,p}(x)\right)=\var\left(\delta_{k-1,p}(x)\right)$ 
	if and only if $\beta_k^2=\frac{2}{n_{k}}.$
    \end{enumerate}
\end{theorem}
Let us prove Theorem \ref{thm:main}.
The following inequalities will be useful.
\begin{lemma}\label{lem:AB}
    Let us define 
    \begin{align}
	A_n=\sum_{k=1}^{n}{n\choose k}\sqrt{k-1},\quad
	B_n=\sum_{k=1}^{n}{n\choose k}\sqrt{k+1}.
    \end{align}
    Then it holds
    \begin{align*}
	&2^n\sqrt{\frac{n}{2}} \left(1-\frac{3}{2n}-\frac{2}{n^2}\right) \leq A_n\leq B_n\leq 2^n\sqrt{\frac{n}{2}+1}.
    \end{align*}
    As a consequence, we have for $n\geq\bigOmg{\epsilon^{-1}}$ that
    \begin{align*}
	(1-\epsilon)\ 2^n\sqrt{\frac{n}{2}} \leq A_n\leq B_n\leq (1+\epsilon)\ 2^n\sqrt{\frac{n}{2}} .
    \end{align*}
\end{lemma}
\begin{proof}
    Let $X\distas{}B(n,1/2)$ be a binomial random variable.
    For every $t\geq 0,$ we have $\sqrt{t}\geq\frac{3t-t^2}{2}.$
    Applying this inequality to $X/(\E X+1)$ and taking the expectation of both sides, we get
    \begin{align}\label{eq:root_X}
	\E\sqrt{X}\geq \sqrt{\E X+1} \left( \frac{3\E X}{2(\E X+1)} - \frac{\E(X^2)}{2(\E X+1)^2}\right) .
    \end{align}
    Note that $A_n=2^n\E\sqrt{X-1}.$ 
    By applying \eqref{eq:root_X} to the random variable $X-1$, we obtain
    \begin{align}
	A_n\geq 2^n\sqrt{\E X} \left( \frac{3(\E X-1)}{2\E X} - \frac{\E(X^2)-2\E X+1}{2(\E X)^2}\right) .
    \end{align}
    Substituting $\E X=\frac{n}{2}$ and $\E(X^2)=\frac{n^2+n}{4}$ gives the result.
    Finally, we have
    \begin{align}
	B_n=2^n\E\sqrt{X+1}\leq 2^n\sqrt{\E X+1} =2^n\sqrt{\frac{n}{2}+1}.
    \end{align}
\end{proof}

Theorem \ref{thm:main} follows directly from the results of Theorem \ref{thm:forward} and Theorem \ref{thm:backward} presented below.
\begin{theorem}[Forward Pass]\label{thm:forward}
    Fix any $k\in[L],p\in[n_k],x\in\RR^{n_0}.$
    Let $S_k$ be defined as in \eqref{eq:Sk}.
    Fix any $\epsilon\in(0,1).$ 
    Suppose that $\min_{l\in[k-1]}n_l\geq \bigOmg{\frac{k}{\log(1+\epsilon)}}.$
    Then, we have:
    \begin{enumerate}
	\item 
	    First moment: $(1-\epsilon) S_k \leq \E[f_{k,p}(x)] \leq (1+\epsilon) S_k.$
	
	\item 
	    Second moment: $\E(f_{k,p}(x)^2)=\frac{\norm{x}_2^2}{2} \left(\prod_{l=1}^{k-1}\frac{n_l}{2}\right) \left(\prod_{l=1}^{k}\beta_l^2\right).$
	
	\item 
	    Variance: $\left(\pi-(1+\epsilon)^2\right) S_k^2 \leq \var(f_{k,p}(x))\leq \left(\pi-(1-\epsilon)^2\right) S_k^2.$
    \end{enumerate}
\end{theorem}
\begin{proof}
    \begin{enumerate}
	\item 
    Let $\mathcal{F}_{l}=\Set{W_1,\ldots,W_l}.$
    Below we omit the argument $x$ as it is clear from the context.
    \begin{align*}
	\E[f_{k,p}]
	=\E[\Econd{\sigma(\inner{f_{k-1},(W_k)_{:p}})}{\mathcal{F}_{k-1}}]
	=\frac{\beta_k}{\sqrt{2\pi}} \E\norm{f_{k-1}}_2.
    \end{align*}
    For convenience, let $\int_{\RR}^{\otimes m}$ denote the $m$-times iterated integral $\int_{\RR}\ldots\int_{\RR}.$
    Let $v_j=\inner{f_{k-2},(W_{k-1})_{:j}}.$ 
    Let $\mathcal{F}_{k-2}=\Set{W_1,\ldots,W_{k-2}}.$
    Conditioned on $\mathcal{F}_{k-2}$, the variables $v_j$'s are independent Gaussian random variables: $v_j\distas{}\mathcal{N}(0,\beta_{k-1}^2\norm{f_{k-2}}_2^2).$
    We have
    \begin{align*}
	\E\norm{f_{k-1}}_2
	&=\bigE{\Econd{\sqrt{\sum_{j=1}^{n_{k-1}}\sigma(v_j)^2 }}{\mathcal{F}_{k-2}}}\\
	&=\bigE{\int_{\RR_{+}\cup\RR_{-}}^{\otimes n_{k-1}} \sqrt{\sum_{j=1}^{n_{k-1}}\sigma(v_j)^2 }\ dP(v_1\vert \mathcal{F}_{k-2}) \ldots dP(v_{n_{k-1}}\vert \mathcal{F}_{k-2}) } \\
	&=\bigE{\sum_{i=1}^{n_{k-1}} {n_{k-1}\choose i} \int_{\RR_{+}}^{\otimes i} \int_{\RR_{-}}^{\otimes (n_{k-1}-i)} \sqrt{\sum_{j=1}^{i} v_j^2 }\ dP(v_1\vert \mathcal{F}_{k-2}) \ldots dP(v_{n_{k-1}}\vert \mathcal{F}_{k-2}) } \\
	&=\bigE{\sum_{i=1}^{n_{k-1}} {n_{k-1}\choose i} 2^{-(n_{k-1}-i)} \int_{\RR_{+}}^{\otimes i} \sqrt{\sum_{j=1}^{i} v_j^2 }\ dP(v_1\vert \mathcal{F}_{k-2}) \ldots dP(v_i\vert \mathcal{F}_{k-2}) } \\
	&=\bigE{\sum_{i=1}^{n_{k-1}} {n_{k-1}\choose i} 2^{-(n_{k-1}-i)}\ 2^{-i} \int_{\RR}^{\otimes i} \sqrt{\sum_{j=1}^{i} v_j^2 }\ dP(v_1\vert \mathcal{F}_{k-2}) \ldots dP(v_i\vert \mathcal{F}_{k-2}) } \\
	&=\sum_{i=1}^{n_{k-1}} {n_{k-1}\choose i} 2^{-n_{k-1}} \E\left[\Econd{\sqrt{\sum_{j=1}^{i}v_j^2}}{\mathcal{F}_{k-2}}\right] \\
	&=\left[\beta_{k-1}\sum_{i=1}^{n_{k-1}} {n_{k-1}\choose i} 2^{-n_{k-1}} \frac{\sqrt{2}\GF{\frac{i+1}{2}}}{\GF{\frac{i}{2}}}\right] \E\norm{f_{k-2}}_2.
    \end{align*}
    Iterating this equality gives
    \begin{align*}
	\E[f_{k,p}]
	&=\frac{\norm{x}_2}{\sqrt{2\pi}} 
	\beta_k \prod_{l=1}^{k-1} \left[\beta_{l} \sum_{i=1}^{n_{l}} {n_{l}\choose i} 2^{-n_{l}} \frac{\sqrt{2}\GF{\frac{i+1}{2}}}{\GF{\frac{i}{2}}}\right].
    \end{align*}
    By Gautschi's inequality, we have
    \begin{align*}
	\sqrt{i-1} \leq\frac{\sqrt{2}\GF{\frac{i+1}{2}}}{\GF{\frac{i}{2}}} \leq\sqrt{i+1} .
    \end{align*}
    This combined with Lemma \ref{lem:AB} yields
    \begin{align*}
	\frac{\norm{x}_2}{\sqrt{2\pi}} \beta_k \prod_{l=1}^{k-1} \left[ \left(1-\frac{\epsilon}{k-1}\right)\ \beta_{l} \sqrt{\frac{n_l}{2}}\ \right]
	\leq \E[f_{k,p}] \leq 
	\frac{\norm{x}_2}{\sqrt{2\pi}} \beta_k \prod_{l=1}^{k-1} \left[ \left(1+\frac{\log(1+\epsilon)}{k-1}\right)\ \beta_{l} \sqrt{\frac{n_l}{2}}\ \right].
    \end{align*}
    where we used twice our assumption in the corollary.
    Using the facts that $1+x\leq e^x$ and $(1-\epsilon)^k\geq 1-k\epsilon$ for $\epsilon\in(0,1)$,
    the final result follows.
    
    \item 
    We have $f_{k,p}(x)=\sigma(\inner{(W_k)_{:p}, f_{k-1}(x)}).$
    Note that the distribution of the inner product is symmetric around $0$, 
    and thus taking the expectation over $W_L$ yields
    \begin{align}
	\E_{\Set{W_k}}(f_{k,p}(x)^2)
	=\frac{\beta_k^2}{2} \norm{f_{k-1}(x)}_2^2
	=\frac{\beta_k^2}{2}\sum_{j=1}^{n_{k-1}} f_{k-1,j}(x)^2
	=\frac{\beta_k^2}{2}\sum_{j=1}^{n_{k-1}} \sigma(\inner{(W_{k-1})_{:j}, f_{k-2}(x)})^2.
    \end{align}
    Taking the expectation with respect to $W_{L-1}$, and using the symmetry again, we obtain
    \begin{align}
	\E_{\Set{W_k,W_{k-1}}}(f_{k,p}(x)^2)=\frac{\beta_k^2}{2}\frac{n_{k-1}\beta_{k-1}^2}{2} \norm{f_{k-2}(x)}_2^2.
    \end{align}
    Iterating this equality leads to the result.
    
    \item This follows from the first two statements.
    \end{enumerate}
\end{proof}

The next lemma characterizes the second moments of the neurons and backward derivatives.
\begin{theorem}[Backward Pass]\label{thm:backward}
    Fix any $k\in[L],p\in[n_k],x\in\RR^{n_0}.$
    Then, we have:
    \begin{enumerate}
	\item First moment: $\E(\delta_{k,p}(x))=0.$
	\item Second moment and Variance:
	\begin{align}\label{eq:Edelta}
	    \var(\delta_{k,p}(x))=\E(\delta_{k,p}(x)^2)=\frac{1}{2} \left(\prod_{l=k+1}^{L-1}\frac{n_l}{2}\right) \left(\prod_{l=k+1}^{L}\beta_l^2\right) .
	\end{align}
    \end{enumerate}
\end{theorem}
\begin{proof}
    Let $v_r$ be the vector defined by 
    \begin{align}
	v_{r}^T=\sigma'(g_{k,p}(x))\ (W_{k+1})_{p:}^T \Sigma_{k+1}(x) \left(\prod_{l=k+2}^{r} W_{l} \Sigma_{l}(x) \right).
    \end{align}
    where $\Sigma_l(x)=\diag([\sigma'(g_{l,j}(x))]_{j=1}^{n_l}).$ 
    By the chain rules, we have
    \begin{align}\label{eq:delta}
	\delta_{k,p}(x)=v_{L-1}^T W_L.
    \end{align}
    \begin{enumerate}
	\item This follows directly from \eqref{eq:delta}.
	\item 
    From \eqref{eq:delta}, we have
    \begin{align}
	\E_{\Set{W_L}} (\delta_{k,p}(x)^2)
	=\beta_L^2\norm{v_{L-1}}_2^2 .
    \end{align}
    By definition, we have $v_{L-1}^T=v_{L-2}^T W_{L-1}\Sigma_{L-1}(x)$, and thus it holds
    \begin{align*}
	\norm{v_{L-1}}_2^2 
	&=\sum_{j=1}^{n_{L-1}} \inner{v_{L-2}, (W_{L-1})_{:j}}^2 \sigma'(\inner{(W_{L-1})_{:j}, f_{L-2}(x)}) .
    \end{align*}
    Let $w$ be a copy of the random vector $(W_{L-1})_{:1}.$ 
    Conditioned on $f_{L-2}(x)$, the RHS of the previous expression is a sum of i.i.d.\ terms, so we have
    \begin{align*}
	\E_{\Set{W_{L-1}}} \norm{v_{L-1}}_2^2
	= n_{L-1} \E_{\Set{w}} \inner{v_{L-2}, w}^2 \sigma'(\inner{w, f_{L-2}(x)})
	=\frac{n_{L-1}\beta_{L-1}^2}{2}\norm{v_{L-2}}_2^2 ,
    \end{align*}
    where the last equality follows from the fact that $w$ and $-w$ have the same distribution, and we used the identity $\sigma'(-x)=1-\sigma'(x).$
    Iterating this argument leads to the result.
    \end{enumerate}
\end{proof}

\bibliography{regul}
\bibliographystyle{plain}
\end{document}